\pdfoutput=1 
\documentclass[11pt]{article}

\usepackage[font=libertinus, citestyle=numeric]{kurbanlab}

\DeclareAffiliation{hbku}{%
  College of Science and Engineering, Hamad Bin Khalifa University, Doha, Qatar}

\DeclareAffiliation{tamu}{%
  Department of Electrical and Computer Engineering,
  Texas A\&M University, College Station, TX, USA}

\DeclareAffiliation{iub}{%
  Luddy School of Informatics, Computing, and Engineering,
  Indiana University Bloomington, Bloomington, IN, USA}

\DeclareAffiliation{tamuq}{%
  Department of Electrical and Computer Engineering,
  Texas A\&M University at Qatar, Doha, Qatar}

\DeclareAffiliation{ankara}{%
  Department of Prosthetics and Orthotics,
  Ankara University, Ankara, Turkey}

\usepackage{colortbl}  

\definecolor{tabverify}{HTML}{1B6E8C} 
\definecolor{tabsub}{HTML}{6B7A87}    
\definecolor{tabhi}{HTML}{EAF1F4}     
\newcommand{\sig}[1]{{\bfseries\color{tabverify}#1}}   
\newcommand{\nullv}[1]{{\color{tabsub}#1}}             
\newcommand{\rowhi}{\rowcolor{tabhi}}                  

\newcommand{\errcommit}{err$\vert$commit}

\title{Grounded verification of chemical and materials reasoning: detection is the bottleneck}
\RunningTitle{Grounded verification of chemical and materials reasoning}

\Author[orcid=0000-0002-1458-302X]{Can Polat}{tamu}
\Author[corresponding=kurbanm@ankara.edu.tr, orcid=0000-0002-7263-0234]{Mustafa Kurban}{tamuq, ankara}
\Author[orcid=0000-0001-9069-770X]{Erchin Serpedin}{tamu}
\Author[corresponding=hkurban@hbku.edu.qa, orcid=0000-0003-3142-2866]{Hasan Kurban}{hbku}

\CodeURL{https://github.com/KurbanIntelligenceLab/grounded-matsci}

\begin{document}
\maketitle

\begin{abstract}
Language models are moving into chemistry and materials discovery workflows, where a wrong molecular formula, space group, or formation energy can silently propagate into downstream decisions. These confabulations hide inside fluent reasoning traces and concentrate on rare, long-tail entities, where model confidence is least trustworthy. Retrieving reference data for every prompt would catch them, but at a heavy coverage and abstention cost. We show that deterministic, database-grounded verification catches and repairs these errors selectively, and that the binding constraint is detection rather than repair. Our tiered verifier extracts each checkable claim, tests it against authoritative databases and physical law, and retrieves a reference value only when a check fails. Across four models and over five hundred prompts with pinned conditions, gated correction cuts the error rate of committed formulas from 22\% to 4\% with 3.2 times fewer retrievals than blanket augmentation, and it outperforms a conversational retrieval oracle when every answer, corrected or not, is scored. When a flag fires, repair almost always succeeds; the benefit reaches the final answer only where the verifier's scope covers it and where long-tail error exists. Checkable claims, checked cheaply, are a practical lever for trustworthy machine reasoning in chemistry.
\end{abstract}

\printkeywords

\section{Introduction}\label{sec:intro}

Large language models increasingly expose their intermediate reasoning, and that reasoning is fluent enough to be persuasive even when it is wrong. In chemistry and materials science the errors take a consequential form. The model commits to a specific object that is verifiable in principle but confabulated in fact, a molecular formula, a space group, a formation energy, embedded in coherent prose that offers little signal it is fabricated, a failure mode documented across chemistry benchmarks and tool-using chemistry agents \cite{mirza2025chembench,bran2024chemcrow}. These errors concentrate on the long tail. Rarely seen entities are the ones models get wrong, a dependence of factual recall on training-corpus frequency established in prior work \cite{kandpal2023large,mallen2023trust} and one we observe across four rarity strata, precisely the regime where model confidence is least trustworthy and an external check is most needed.

Interventions that need no external truth cannot fix this. Prompting a model to critique and revise its own answer cannot supply a fact it never knew, and such loops often fail to improve accuracy or degrade it \cite{huang2024selfcorrect,gou2024critic}. What is missing is an external, authoritative signal, a diagnosis shared by work that grounds claims against retrieved evidence rather than against the model itself \cite{min2023factscore}. The two established ways of supplying that signal each fall short of it. Retrieval-augmented generation \cite{lewis2020rag} provides evidence unconditionally but pays a coverage and abstention cost and cannot catch a self-contradiction whose entity is not known a priori. Verification by a second model acting as judge \cite{zheng2023judge} is flexible but unauditable and itself prone to overconfidence.

We therefore take a third route (Fig.~\ref{fig:method}a): a tiered, deterministic verifier that grounds each extractable claim against authoritative databases (PubChem \cite{kim2025pubchem}, Materials Project \cite{jain2013materialsproject}, CCCBDB \cite{johnson2022cccbdb}) and physical law, resolves each claim at the cheapest tier able to decide it (Fig.~\ref{fig:method}b), issues a hard verdict wherever an extractor is reliable, and feeds the reference value back into a gated correction loop. The check is deterministic and auditable rather than generative. On a flagged claim it returns the reference value itself as the repair signal, so correction supplies a fact the model never knew rather than resampling the model's own uncertainty, and every flag traces to a database record or a physical law.

Here we show, evaluating four open-weight and proprietary models under a five-condition protocol on a corpus of five-hundred-plus prompts pinned to explicit thermodynamic and structural conditions, that gated correction is detection-limited rather than correction-limited. Repair succeeds almost wherever a flag fires, so in-loop detection recall, not the correction machinery, is what differentiates verification surfaces and what engineering effort should target. A consistency-triggered second detection stage recovers most of the missed recall where values are extractable yet yields no accuracy lift, because it flags without a reference to repair with. Gated correction also beats a conversational oracle retriever on the intention-to-treat measure a deployed system is judged by, at substantially lower retrieval cost, because conversational fact-provision often leaves the reference value unstated in an extractable, subject-bound form; an extraction-enforcing oracle variant is complementary, winning on copyable quantitative claims where flags are missed. The lift is governed by detection recall rather than by model openness or capability tier, and it concentrates where the committed object is extractable and baseline error exists (Fig.~\ref{fig:method}c).

Two boundaries define the reach of the benefit. Object-level grounding improves object accuracy and calibration, but it lifts the final answer only when the verifier's scope extends to the answer-bearing derived quantity, which we demonstrate by extending the tier ladder and restoring the lift. The method ports across domains through its orchestration layer alone, and the gain appears exactly where extractable long-tail error exists, absent on physical constants where models sit near ceiling and large on long-tail isotope half-lives and materials formulas, with a reference-frame-artifact analysis for quantitative properties recurring across domains. The corpus, frozen ground truth, and frozen pipeline are released with all registration hashes, so every headline number is reproducible under one-shot-holdout discipline.

\begin{figure}[!t]
\centering
\includegraphics[width=\textwidth]{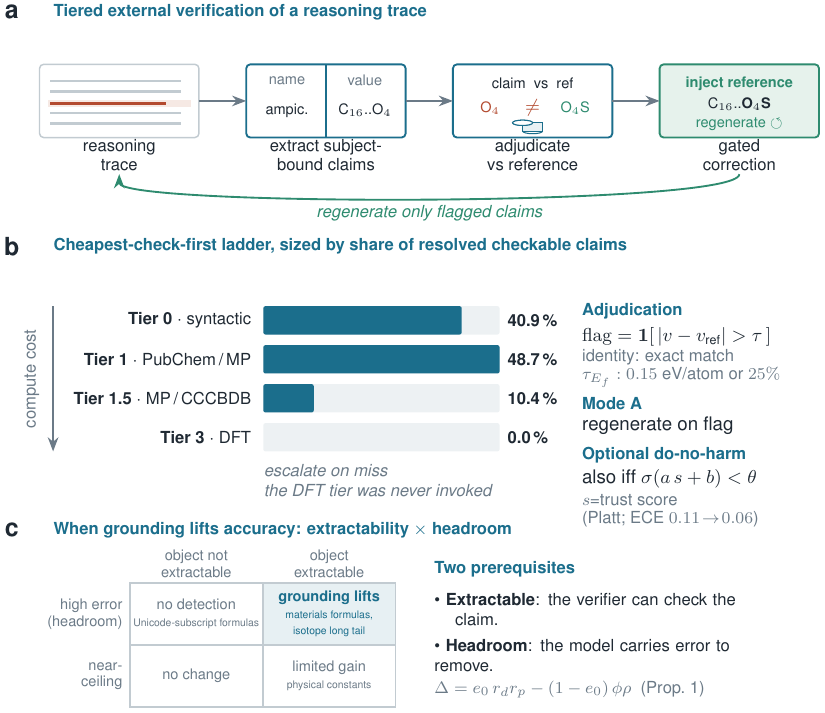}
\caption{Method overview. \textbf{a}, Verify-and-correct loop (ampicillin schematic): extract
(name,~value) claims, check them against PubChem/Materials Project/CCCBDB, and on flag re-inject the reference in
verifier feedback and regenerate, repairing flagged claims only. \textbf{b}, Cheapest-first tier ladder with per-tier holdout resolution shares (density-functional-theory (DFT) tier unused).
Quantitative flags use $\mathrm{flag}=\mathbf{1}[|v-v_\mathrm{ref}|>\tau]$
($\tau_{E_f}=0.15$~eV/atom or $25\%$; identity exact match); Mode~A regenerates on flag, with an
evaluated Platt trust gate on the constants rerun. \textbf{c}, Lift requires extractability
$\times$ headroom (materials formulas and isotope long tail); near-ceiling or unparseable cases show
limited or no gain (Prop.~1).}\label{fig:method}
\end{figure}
\section{Results}\label{sec:results}
\providecommand{\verify}{\textsuperscript{\textcolor{red}{[verify]}}}

\subsection{Confabulation exists and scales with rarity}\label{subsec:rarity}
Molecular-formula error at baseline rises monotonically with compound rarity (PubChem-CID-ordered strata), from 0\% on well-known compounds to 58\% on recent ones (baseline run, pooled over models; per-stratum values in Fig.~\ref{fig:overview}a). Recent drugs (lenacapavir, sotorasib, ensitrelvir) are where models confabulate formulas
most, exactly the long-tail objects a database verifier is positioned to catch. The full error
taxonomy (by model $\times$ claim type $\times$ stratum, with representative examples) is in
Supplementary Table~\ref{tab:errors}.

\subsection{Gated grounding closes the gap and beats the oracle on the deployment metric}\label{subsec:gated}
On molecular formula, gated correction (Mode~A) drives error-given-commitment from 22\% to 4\% while retrieving only on flagged cells, $3.2\times$ fewer retrievals than blanket retrieval-augmented generation (RAG) (Supplementary Table~\ref{tab:compute} and Supplementary Fig.~\ref{fig:cost}). On intention-to-treat (ITT) error, wrong or non-committal under the frozen extractor and the metric a user experiences, gated Mode~A reaches 25\% against \emph{conversational} oracle RAG's 43\%: handed the fact but not told to restate it, the model often answers in prose without emitting an extractable, subject-bound value (42\% no-commit), whereas Mode~A's direct correction question holds no-commit at the baseline 22\% (Fig.~\ref{fig:overview}b). The oracle framing is an upper bound on retrieval quality; a realistic, non-oracle retriever resolving each entity live is characterised in Supplementary Note~\ref{note:nonoracle}.

That no-commit is largely a phrasing artifact rather than model abstention. A second oracle variant forcing
explicit restatement of the retrieved value (constrained RAG) recovers it: 73\% of the constrained
variant's no-commits, and 65\% of the conversational variant's, are cells where the correct formula appears
in the text but is not bound to the subject in the pattern the frozen extractor requires. We therefore
report those surfaces as no-commit-under-the-frozen-extractor rather than as abstention, and the
extractability limit this exposes is the same axis that governs where grounding pays off. Where the
deliverable is a single copyable number the confound vanishes: on formation energy, constrained RAG reaches
ITT error 1.0\% against conversational RAG's 25\% and gated Mode~A's 33\%. Forcing restatement of a
retrieved scalar solves the surface Mode~A cannot, because Mode~A is detection-limited there, so the two are
complementary, gated correction winning on cost and on identity claims and extraction-enforcing retrieval on
copyable quantitative values.

Both oracle conditions are upper bounds, and replacing the oracle with live entity resolution shows the
gain survives only where entities are name-addressable. Realistic retrieval matches the oracle on
molecular formula (0.2\% versus 0.3\% error given commitment against the 22\% baseline) but recovers
almost none of it on the long-tail crystalline and formation-energy surfaces, where coverage reaches only
about a fifth of entities and the retriever returns the most-stable polymorph rather than the named phase,
so abstention rises to 56--59\% as the model, handed no usable fact, declines to commit. Retrieval
degrades precisely where the named entity requires phase disambiguation, the regime in which the
verifier's claim-specific retrieval still succeeds.

\begin{figure}[htbp]
\centering
\includegraphics[width=\textwidth]{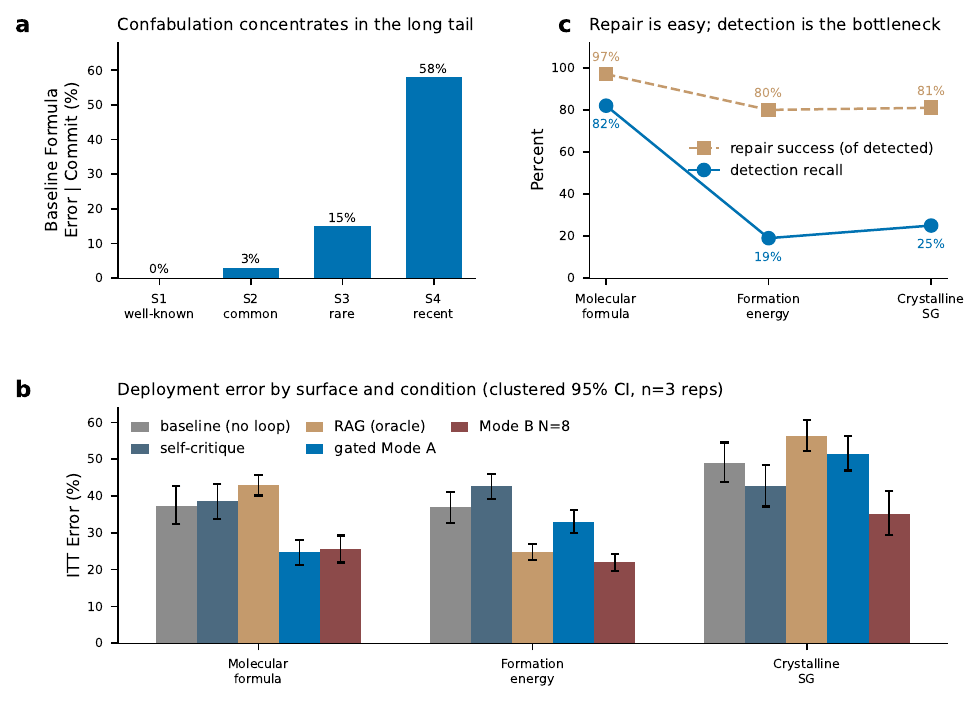}
\caption{The problem and the fix. \textbf{a}, Baseline molecular-formula error given commitment by compound-rarity stratum. \textbf{b}, Intention-to-treat error (wrong or no-commit) by surface and condition; error bars, prompt-clustered 95\% confidence intervals; Mode~A, gated correction; Mode~B, best-of-8 rerank. \textbf{c}, Two-stage decomposition of gated correction: repair success and in-loop detection recall by surface.}\label{fig:overview}
\end{figure}

The pooled formula effect is not uniform across models, and the per-model decomposition in
Table~\ref{tab:permodel} shows why: the lift lands on the two open-weights models, which carry both
extractable formulas and real long-tail error, while the two frontier closed models start near
ceiling and have almost no formula headroom for a verifier to remove.

\begin{table}[htbp]
\caption{Per-model lift on molecular formula and formation energy (baseline $\rightarrow$ Mode~A \errcommit; McNemar net corrections on aggregated units). Teal marks a significant gain; grey marks a non-significant change.}\label{tab:permodel}
\footnotesize
\setlength{\tabcolsep}{4pt}
\renewcommand{\arraystretch}{1.18}
\begin{tabular*}{\textwidth}{@{\extracolsep{\fill}}ll cc c cc@{}}
\toprule
& & \multicolumn{3}{c}{\textbf{Molecular formula}} & \multicolumn{2}{c}{\textbf{Formation energy}} \\
\cmidrule(lr){3-5}\cmidrule(l){6-7}
model & class & \errcommit & net & $p$ & net & $p$ \\
\midrule
Claude Sonnet 5 & closed & 10\%\,$\rightarrow$\,8\% & \nullv{$+1$} & \nullv{n.s.} & \nullv{$+5$} & \nullv{n.s.} \\
GPT-5.5 & closed & 1\%\,$\rightarrow$\,0\% & \nullv{$0$} & \nullv{n.s.} & \nullv{$+3$} & \nullv{n.s.} \\
\rowhi DeepSeek v4 Pro & open & 20\%\,$\rightarrow$\,2\% & \sig{$+23$} & $5.6\!\times\!10^{-6}$ & \sig{$+18$} & $7.6\!\times\!10^{-6}$ \\
\rowhi Qwen3-235B & open & 46\%\,$\rightarrow$\,1\% & \sig{$+58$} & $6.9\!\times\!10^{-18}$ & \nullv{$+1$} & \nullv{n.s.} \\
\bottomrule
\end{tabular*}
\tabnote{net, McNemar net corrections (fixed $-$ broken); \errcommit, error-given-commitment;
class, closed = frontier closed-weights, open = open-weights; n.s., not significant ($p>0.05$).}
\end{table}

The formation-energy column tells the complementary story. Qwen3-235B has the largest formation-energy (EF) headroom of any model (59\% baseline error-given-commitment) yet shows no EF lift, because its in-loop EF detection recall is $0/57$: the verifier never catches those errors during generation, so there is nothing to repair, and headroom without detection yields no lift. On a domain where the frontier closed models do carry long-tail error, isotope half-lives, the same loop lifts both significantly, so the method is a reliability layer wherever a model is weakest on the long tail, frontier or open.

The operative axis is detection recall, not model openness. Adding two small closed frontier-family models
(Claude Haiku 4.5, GPT-5.4 mini; verified live, full corpus, triplicate, registered ex ante) does not
reproduce the open-weights pattern: GPT-5.4 mini shows a small non-significant molecular lift, and Claude
Haiku 4.5 carries a high baseline molecular error (50\%) that the loop leaves untouched because it emits
formulas in Unicode subscripts the frozen extractor cannot parse, driving in-loop recall to approximately
zero. The operative variable is therefore detection recall, whether the committed object is extractable,
which interacts with output formatting, multiplied by baseline error, rather than open-versus-closed
weights or capability tier.

\subsection{An error-reduction identity for gated correction}\label{subsec:theory}
Gated correction acts on a fixed population of committed, gradable claims, moving each through a
two-outcome tree, so its effect on error-given-commitment decomposes exactly. Fix a surface and the
population of claims committed and graded in both the baseline and Mode-A arms. Let $e_0$ be the baseline
error-given-commitment; $r_d$ the in-loop detection recall, the probability that a baseline-wrong claim is
flagged and regenerated during Mode-A generation; $r_p$ the repair success, the probability that a flagged
wrong claim becomes correct once the reference value is injected; $\phi$ the false-positive rate; and
$\rho$ the regeneration-break rate, the probability that such a regeneration turns a correct claim wrong.
All five are conditional rates estimated directly from the runs, and where a Platt do-no-harm gate is
active it is folded into $r_d$ and $\phi$.

\begin{proposition}\label{prop:identity}
On the committed population, the post-correction error-given-commitment is
\begin{equation}
e_1 \;=\; e_0\,(1-r_d r_p) \;+\; (1-e_0)\,\phi\rho,
\label{eq:identity}
\end{equation}
\end{proposition}
so the net reduction is
\begin{equation*}
\Delta = e_0-e_1 = e_0\,r_d r_p - (1-e_0)\,\phi\rho.
\end{equation*}

\noindent\emph{Proof sketch.} A baseline-wrong claim stays wrong unless it is both detected-and-gated and
repaired, probability $r_d r_p$; a baseline-correct claim turns wrong only if it is both falsely
flagged-and-gated and broken on regeneration, probability $\phi\rho$. Summing over the mass-$e_0$ and
mass-$(1-e_0)$ subpopulations gives \eqref{eq:identity} by the law of total probability. The complete
argument, and the form expressing the harm term through detector precision, are in Supplementary
Note~\ref{note:identity}.

The identity turns three descriptive findings into exact statements using only quantities the experiments
already report. \emph{Detection-limited:} repair is high and near-uniform across surfaces
($r_p\in[0.80,0.97]$) while recall varies widely ($r_d\in[0.19,0.94]$), so the across-surface spread of
$\Delta$ is carried by $r_d$. \emph{Headroom-gated and do-no-harm:} the gain term scales with $e_0$, so a
near-ceiling surface can only lose, $\Delta=-(1-e_0)\phi\rho\le0$, which is why the frontier models incur
small net harm on physical constants and why a calibrated gate lowering $\phi$ is the natural do-no-harm
instrument. \emph{Reference-necessity:} a detector that flags without injecting a reference replaces $r_p$ by
the model's unaided resample-correct rate, near zero for a long-tail identity it never learned, so recall can
climb with no accuracy lift, exactly the consistency-stage result below. Substituting the measured
molecular-formula rates ($e_0=0.22$, $r_d=0.82$, $r_p=0.97$) predicts $4.5\%$ against the observed $4\%$.

\subsection{The mechanism is two-stage: repair is easy, detection is the bottleneck}\label{subsec:mechanism}
Gated correction is detect, then repair, and the two stages behave very differently. Wherever the verifier flags an error, the loop repairs it at 80--97\% whether the surface is molecular formula, formation energy or space group (SG): the model readily accepts an injected reference fact regardless of surface. The lift a surface receives is therefore set almost entirely by in-loop detection recall, the fraction of baseline-wrong prompt$\times$model units whose flaggable claim is caught during Mode~A generation, which is high on formula and low on formation energy and space group; both stages are decomposed per surface in Fig.~\ref{fig:overview}c. In-loop recall runs below offline recall, where the same verifier scores frozen baseline claim text, because in the loop the model must re-emit the flaggable claim, and generation variance, claim-binding context and gating timing all reduce what the verifier ever sees. The in-loop number is the
deployment-relevant one and is used throughout; the offline number bounds the extractor. Formation energy,
even after a dedicated detection tier raised in-loop recall from a tier-off 7\% to 19\%, gets only a small
though significant lift (31\% to 24\%), and the residual is precisely the undetected
errors rather than failed repairs.

A hybrid detection stage raises recall without raising accuracy, and localizes why. Adding a
consistency-triggered second stage to Mode~A as a two-stage triage, with the deterministic tier firing first
and the sampling-consistency signal invoked at its dev-frozen threshold only on units the deterministic tier
misses, raises in-loop EF detection recall from 19\% to 94--95\% where values are extractable, yet
error-given-commitment is flat-to-worse. Two mechanisms explain this. The consistency signal has poor
precision on this surface, flagging 94\% of Claude Sonnet 5's EF cells including 94\% of the correct ones
(precision 28\%), because EF values scatter across samples whether or not the committed value is right. And it
signals unreliability without injecting a reference value, unlike the deterministic tier which carries the
database fact, so only 46\% of regenerated flagged cells become correct: the model resamples another wrong
value it never learned. Detection and repair are both necessary, and a detector that flags without a fact
raises recall alone, so we report the recall gain and make no causal accuracy claim.

\subsection{Rerank and correction are complementary, indexed by where the truth lives}\label{subsec:rerank}
Best-of-8 rerank (Mode~B) beats gated correction on formation energy (\errcommit\ 21\% versus 24\%, ITT 22\%
versus 33\%), the complement of the identity-claim result, and the reason is where the
truth lives. For a property the model has seen but samples noisily, eight samples surface the correct value
from its own distribution; for a long-tail identity it never learned, no number of samples contains the fact
and only external correction helps. Those identity errors do surface as cross-sample disagreement, but
disagreement signals only that the model is unsure: the samples scatter across wrong values with no correct
one among them, making it a reliable flag that points to no answer. Across surfaces Mode~B is best-or-tied on
ITT everywhere and wins crystalline and formation energy, while Mode~A wins molecular formula on \errcommit\
and on cost; blanket RAG worsens crystalline ITT, because provided facts raise abstention rather than
accuracy on a surface the models rarely commit on. Mode~B's gain carries a compute cost: best-of-8 issues
about $8\times$ the baseline generation calls against gated Mode~A's ${\sim}1.75\times$.

\subsection{Self-correction without tools does not help and marginally harms}\label{subsec:selfcritique}
Self-critique (no tools) improved no surface, moving molecular-formula and crystalline error by at most one point, and marginally harmed formation energy (31 to 38\% \errcommit: 55 correct answers broken vs 35 fixed), consistent with prior evidence that language models do not reliably self-correct reasoning without an external signal. A stronger tool-free baseline, Chain-of-Verification \cite{dhuliawala2024cove}, in which the model plans and independently answers its own verification questions before revising, gives the same result and closes no surface (Supplementary Note~\ref{note:cove}). This is the negative control: the lift under tool-grounded conditions is attributable to external truth, not to the extra reasoning turn.

\subsection{Naive single-frame property verification is dominated by reference-frame artifacts}\label{subsec:frames}
A verifier checking a quantitative property against a single reference systematically mislabels correct
answers when model and reference use different physical frames. Our Materials Project band-gap ground truth is
PBE-DFT, which underestimates experimental gaps by 1--2~eV, while models report experimental gaps: a naive
verifier flagged 116 band-gap ``errors'', of which 107 (92\%) were correct experimental values, and under an
accept-either-frame policy band-gap error falls from 34.6\% to 0\%. Band gap is not an error surface for these
models; it is a frame-mismatch detector. Formation energy is the control, since 0~K DFT and experimental
$\Delta H_f^{\circ}$ at 298~K differ by only 0.11~eV/atom on average, so accept-either-frame rescued just 9
cells and formation energy remains a genuine error surface. The same lesson recurred for physical constants,
where the Sackur--Tetrode constant is tabulated at two standard-state pressures and models reporting the modern
value were initially graded wrong against a reference using the older convention, cured only by pinning the
standard state in the ground truth. Property verification must therefore register the reference frame ex ante,
grade under accept-either-documented-frame where frames diverge, and never ``correct'' an experimental value to
a DFT one (Supplementary Table~\ref{tab:frames}).

\subsection{Deterministic verifier vs LLM-as-judge vs sampling-based detectors}\label{subsec:detectors}
We contrast four detectors on the same baseline claims: the deterministic verifier; a frontier
reasoning-class LLM-as-judge with no tools; a sampling-consistency baseline
\cite{manakul2023selfcheckgpt} flagging a claim when the object extracted from Mode~B's eight samples
disagrees (modal-object agreement $\le 0.25$, dev-frozen); and a semantic-entropy detector
\cite{farquhar2024semantic} clustering those samples by meaning and flagging high normalized cluster
entropy (threshold $0.44$, likewise dev-frozen). Supplementary Table~\ref{tab:detectors} gives per-surface
precision and recall for all four and shows why the deterministic checker is the backbone: only it returns
a reference value to inject. These offline recalls, scored on frozen baseline claim text, upper-bound the
extractor and run systematically above the in-loop recalls, the gap reflecting generation variance and
gating timing rather than a different verifier.

Each detector has a distinct strength. The deterministic verifier dominates recall where its extractor is
reliable and is auditable, improvable and frozen. The judge is precise where symbolic
extraction is hard but is a fixed black box subject to the same confabulation it is meant to catch.
Sampling-consistency has the strongest recall on formation energy (0.67), exactly where the deterministic
verifier is weakest (0.31), because noisy-sampling errors surface as cross-sample disagreement. The predicted
failure mode, that consistency should miss confidently-wrong long-tail formulas because all eight samples
agree, does not hold: on rare (R3--R4) molecular formulas consistency recall is 86\% at mean agreement 0.21,
so models disagree across samples on long-tail objects rather than confabulating one value consistently.
Semantic entropy refines that signal, improving formation-energy precision (0.51 versus 0.36) by clustering
near-duplicate values by meaning rather than by exact string, at slightly lower recall. Like
sampling-consistency it is detection-only and returns no reference value, the axis that separates both from
the deterministic tier.

\subsection{Trust-score calibration}\label{subsec:calibration}
The verifier's trust score, the fraction of checked claims that pass, predicts committed-claim
correctness and is better calibrated than the model's own confidence: on the paired population verbalized confidence has expected calibration error (ECE) 0.139 against the trust score's
0.093, and Platt scaling \cite{platt1999} on a held-out half reduces the trust score to 0.060
\cite{guo2017calibration}. Eliciting confidence inline rather than post hoc makes the model more
confident without making it better calibrated. The grounded signal is therefore the better abstention
instrument, and the calibrated form is what the do-no-harm gate thresholds on; the reliability
diagrams, the full ECE and Brier values, and the isotonic comparison are in Supplementary Note~\ref{note:calibration}.

\subsection{Downstream end-task propagation}\label{subsec:endtask}
To test whether claim-level grounding translates into end-task accuracy, we registered 52 tasks, frozen ex
ante, whose final answer depends on a checkable object, molar-mass computation from formula, higher-molar-mass selection and more-stable-phase selection from formation energy, weighted toward the rare strata (counts in Table~\ref{tab:endtask}). Object errors propagate at 9\% of baseline task units. Object-only gated grounding
nonetheless did not lift end-task accuracy: ITT accuracy was 83\% at baseline, 80\% under Mode~A and 85\%
under Mode~B, the Mode-A change not significant ($p\approx0.6$; Table~\ref{tab:endtask}). The mechanism is diagnostic. The
verifier detects objects but fires largely on incidental claims the system prompt elicits (flag rate 58\%, of
which an audit finds 203 of 205 bear on no part of the answer), fixes only 28\% of propagated errors because
the final answer depends on downstream arithmetic the verifier does not check, and adds regeneration noise
that breaks 8\% of already-correct answers. This is a grounding-scope mismatch, not an extractor robustness
failure.

That diagnosis makes a testable prediction, and closing the scope gap confirms it. We added an
orchestration-layer derived-quantity tier, registered ex ante, computing the answer-bearing quantity from the
verified object and injecting it when the model's final answer disagrees. Re-running the same 52 tasks, pooled end-task ITT accuracy rose significantly, from 83\% to 90\%, against 80\% under the object-only condition (Table~\ref{tab:endtask}). The lift concentrates in one open-weights model, Qwen3-235B, while the two frontier models have no headroom and DeepSeek v4 Pro shows small non-significant harm; by task type it falls on molar-mass (86\% to 95\%). Grounding scope is therefore the
relevant boundary: verifying the checkable object improves object accuracy and calibration, whereas extending
verification to the answer-bearing derived quantity is what improves the end task.

\begin{table}[htbp]
\caption{End-task grounding (52 registered tasks, frozen ex ante; 4 models, triplicate). \textbf{Top}, baseline error propagation by task type. \textbf{Bottom}, end-task ITT accuracy, baseline versus the added orchestration-layer derived-quantity tier, per model. Shading marks statistically significant results.}\label{tab:endtask}
\footnotesize
\setlength{\tabcolsep}{4pt}
\renewcommand{\arraystretch}{1.18}
\begin{tabular*}{\textwidth}{@{\extracolsep{\fill}}l cc c@{}}
\toprule
\multicolumn{4}{@{}l}{\itshape Baseline error propagation, by task type} \\
task type & \multicolumn{2}{c}{$n$} & propagated error \\
\midrule
molar-mass computation from formula & \multicolumn{2}{c}{26} & 12\% \\
higher-molar-mass selection & \multicolumn{2}{c}{14} & 7\% \\
more-stable-phase selection (formation energy) & \multicolumn{2}{c}{12} & 4\% \\
overall & \multicolumn{2}{c}{52} & 9\% \\
\addlinespace[2pt]
\midrule
\multicolumn{4}{@{}l}{\itshape Object-only vs.\ {+}derived-quantity tier (per model, end-task ITT accuracy)} \\
model & baseline & {+}derived tier & net ($p$) \\
\midrule
Claude Sonnet 5 & 92\% & 94\% & \nullv{n.s.} \\
GPT-5.5 & 71\% & 73\% & \nullv{n.s.} \\
DeepSeek v4 Pro & 96\% & 92\% & \nullv{$-2$, n.s.} \\
\rowcolor{tabhi} Qwen3-235B & 73\% & \sig{100\%} & \sig{$+14$}~{\footnotesize($1.2\!\times\!10^{-4}$)} \\
\rowcolor{tabhi} pooled & 83\% & \sig{90\%} & \sig{$+14$}~{\footnotesize($0.0066$)} \\
\bottomrule
\end{tabular*}
\tabnote{net, McNemar net corrections (pooled: 5 broken / 19 fixed); ITT accuracy, fraction of all evaluated units with a correct final answer, non-committal outputs counted as incorrect; n.s., not significant ($p>0.05$).}
\end{table}

\subsection{Cross-domain transfer: a near-ceiling domain and a headroom domain}\label{subsec:transfer}
To test whether the method template (extract checkable object $\rightarrow$ tiered identity
lookup $\rightarrow$ gated repair) transfers beyond materials chemistry, we applied it to two
further domains chosen to bracket the headroom axis: physical constants (where current
models are near-ceiling, testing whether the mechanics port) and isotope half-lives (where
a genuine long-tail error mass remains, testing whether the accuracy lift ports). Only the
orchestration layer changed in each; the core extractor and verifier configuration
were untouched.

In the near-ceiling domain we graded 35 CODATA 2022 recommended values \cite{mohr2025codata} across four
rarity strata under a registered two-part tolerance policy, so neither the adjustment version nor
quoted-digit truncation can register as an error. The audit of this domain is itself a finding: a first pass
reported a large lift (84\% to 99\%) that a line-by-line audit traced almost entirely to an
orchestration-layer value-parser artifact (Supplementary Note~\ref{note:constants}). After hardening the parser and re-grading, baseline and gated accuracy are statistically indistinguishable, with no model changing significantly (Table~\ref{tab:transfer}). The mechanics port cleanly, but a near-ceiling
domain cannot on its own provide an independent accuracy-lift result. That audit prompted a materials
retro-audit for the same failure mode, which confirmed the materials extractor abstains symmetrically on
unparseable formulas, so the materials lift cannot be parse-manufactured.

The headroom domain supplies the lift result the constants cannot: half-lives of 63 rarity-stratified
isotopes, from famous radioisotopes to superheavy and exotic, graded against the IAEA NDS reference
\cite{iaea_nds} with ground truth and the isotope tier frozen and hashed before grading (4 models,
triplicate). Applying the constants lesson pre-emptively, we hardened the value parser before grading and
registered the unit policy ex ante. This mattered: a first-pass parser mis-read comma-grouped values and
$N\times10^{M}$ forms, inflating baseline error to 40\% and the effect to net $+50$, a constants-class
artifact the two-channel gate caught before it reached the paper. Under Mode~A, ITT accuracy rose significantly from 81\% to 91\% (Table~\ref{tab:transfer}) while error-given-commitment fell from 11.3\% to 0.0\% at symmetric abstention, so the lift is not parse-manufactured. It follows the same rarity gradient as materials, near-ceiling on famous isotopes and largest on the superheavy and exotic stratum, where \errcommit\ falls from 41\% to 0\% (Table~\ref{tab:transfer}).

\begin{table}[htbp]
\caption{Transfer to two non-materials domains under orchestration-layer changes only (frozen core untouched). \textbf{Top}, physical constants (CODATA 2022, post-audit), ITT accuracy per model. \textbf{Bottom}, isotope half-lives (IAEA NDS, hardened parser), ITT accuracy per rarity stratum. Shading marks statistically significant results.}\label{tab:transfer}
\footnotesize
\setlength{\tabcolsep}{4pt}
\renewcommand{\arraystretch}{1.18}
\begin{tabular*}{\textwidth}{@{\extracolsep{\fill}}ll cc c@{}}
\toprule
& & baseline & Mode~A & net ($p$) \\
\midrule
\multicolumn{5}{@{}l}{\itshape Physical constants against CODATA 2022 (per model, ITT accuracy)} \\
Claude Sonnet 5 & frontier closed & 100.0\% & 97.1\% & \nullv{n.s.} \\
GPT-5.5 & frontier closed & 100.0\% & 94.3\% & \nullv{n.s.} \\
DeepSeek v4 Pro & open-weights & 91.4\% & 97.1\% & \nullv{n.s.} \\
Qwen3-235B & open-weights & 88.6\% & 94.3\% & \nullv{n.s.} \\
overall & & 95.0\% & 95.7\% & \nullv{$+1$ (1.0)} \\
\addlinespace[2pt]
\midrule
\multicolumn{5}{@{}l}{\itshape Isotope half-lives against IAEA NDS (per rarity stratum, ITT accuracy)} \\
R1 famous & \hfill 15 & 98\% & 99\% & \nullv{--} \\
R2 common/applied & \hfill 16 & 95\% & 99\% & \nullv{--} \\
R3 rare & \hfill 16 & 88\% & 90\% & \nullv{--} \\
\rowcolor{tabhi} R4 superheavy/exotic & \hfill 16 & 44\% & \textbf{76\%} & \nullv{--} \\
\rowcolor{tabhi} overall & \hfill 63 & 81\% & \sig{91\%} & \sig{$+30$}~{\footnotesize($2.3\!\times\!10^{-7}$)} \\
\bottomrule
\end{tabular*}
\tabnote{ITT accuracy, fraction of all evaluated units with a correct final answer, non-committal outputs counted as incorrect; net, McNemar net corrections; n.s., not significant. The second column gives model class in the top block and isotope count in the bottom block; ``--'' indicates that no per-row significance test was performed.}
\end{table}

The frontier-model lift on this surface is suggestive rather than confirmatory. It is nominally significant
for both frontier-closed models, Claude Sonnet 5 88\% to 99\% and GPT-5.5 82\% to 93\%, but neither $p$-value survives correction for multiple comparisons, and
re-estimating the arm at five replicates leaves both the net lifts and that framing unchanged (Supplementary Note~\ref{note:isotope}). The open-weights lift on the same surface is stronger, with Qwen3-235B 70\% to
85\% surviving correction while DeepSeek v4 Pro is directional. Isotope half-lives are
nonetheless the first surface here on which grounding reaches a frontier closed model at all, because it is
the first on which those models carry genuine long-tail error mass rather than sitting at ceiling: when a
frontier model errs on the long tail and the committed value is extractable, gated correction can reach it,
a reading these data support without establishing at corrected significance.

Across the three transfer domains, then, the mechanics port with orchestration-layer changes only, while the accuracy benefit lands only where a base model still carries extractable error on rarely seen entities: nowhere on the near-ceiling constants, strongly on the isotope and materials-formula long tails. Portability is of the mechanism; the magnitude of benefit is set by headroom $\times$ detection recall (Proposition~\ref{prop:identity}).

The interventions differ sharply in compute cost, which bears on deployment. Gating is the cheap regime: gated Mode~A retrieves on only the flagged fraction of prompt$\times$model cells, and the retrieval saving over blanket RAG reported above widens to $6.2\times$ and $8.4\times$ fewer on formation energy and space groups, at a modest generation-cost overhead, whereas best-of-8 rerank is the expensive regime, both quantified per condition in the compute accounting. The tier ladder keeps per-check cost low: 40.9\% of checkable claims resolve at the Tier-0 syntactic check, 48.7\% at database identity lookups and 10.4\% at tabulated references, with the DFT tier invoked on no holdout claim.
\section{Discussion}\label{sec:discussion}
The results reframe verifier-in-the-loop grounding as a detection-limited rather than a
correction-limited problem, and place it against its two default alternatives. Against blanket
retrieval, gated correction is markedly cheaper on molecular formula and better on the
intention-to-treat deployment metric, because it does not induce the abstention that conversational
fact-provision does. Against LLM-as-judge verification it is auditable and frozen wherever extraction
is reliable, and we report the two side by side rather than assuming either dominates. The
selective-prediction view \cite{elyaniv2010selective} unifies the metrics: error-given-commitment and
abstention trade error against coverage, and we report both the ITT and selective views without
collapsing them. Rerank and correction remain complementary tools.

A deployment caveat follows directly from the near-ceiling models: when the
base answer is already correct, a flag that fires anyway triggers a regeneration that can break it, which
is why the two frontier models show small net harm on the constants domain. Decomposing that harm made the
fix concrete, because the harm and lift populations separate by flag correctness: every frontier harm cell
was a false-positive flag on an already-correct constant, while all open-model lift came from
true-positive flags. A gated rerun using corrected reference values and the dev-frozen Platt-trust
threshold reduced frontier harm from five cells to zero while largely retaining the open-weights lift,
taking overall harm from seven cells to two. Notably that harm was eliminated by correcting a stale
Sackur--Tetrode reference convention rather than by the confidence gate, which suppressed no
regenerations: on a single-claim surface the trust score is binary and perfectly correlated with the
verifier flag, so validating confidence-weighted suppression requires a multi-claim setting with
non-binary trust. The calibrated trust score remains the promising do-no-harm signal on the strength of
its calibration advantage over model self-report.

\subsection*{Our verifier sits between three lines of work} Retrieval-augmented generation prepends
retrieved facts unconditionally, whereas we retrieve only on a detected flag, which is both cheaper and,
on the deployment metric, more accurate. Self-correction and self-critique
\cite{madaan2023selfrefine} revise using the model's own judgment; our no-tools condition reproduces
that negative result and isolates the external signal as the active ingredient. Sampling-based
hallucination detection flags low cross-sample agreement, and we include it as a third detector, finding
it complementary and strongest exactly where symbolic extraction is weakest. Related uncertainty-based
detectors, semantic entropy and internal-state probes, work from the same cross-sample or
representational signal but likewise report only that a claim is unreliable and return no reference
value to inject, so in our decomposition they remain detection-only. Our two-stage triage relates to
detector aggregation, which we implement not to maximize F1 but as a causal instrument that raises
in-loop recall; to two-stage uncertainty triage, which we invert so the deterministic tier fires first
and the consistency signal only on its misses; and to stepwise consistency filtering, which we apply to
committed objects rather than to reasoning steps.

Within chemistry and materials specifically, several recent efforts are adjacent but scoped differently.
Process-level deterministic benchmarks \cite{chemcotbench,moleculariq} verify reasoning steps
deterministically but are evaluation-only and molecular-only, with no correction loop, no crystalline surface
and no cost accounting. Retrieval and knowledge-graph hallucination work for materials \cite{hallumat}
grounds against retrieved documents or graphs, whereas ours grounds against deterministic physics and
structured databases and reports the end-task and cost consequences rather than detection alone. Tool-using
chemistry agents \cite{bran2024chemcrow} and LLM chemistry benchmarks \cite{mirza2025chembench} document the
tool-integration and overconfidence problems and the failure of LLM-judge scoring; we quantify that
overconfidence through calibration and avoid LLM-judge labelling entirely. Agentic self-verification instead
places the check inside the model: xChemAgents \cite{polat2025xchemagents} pairs a descriptor-selecting agent
with a Validator agent enforcing unit consistency and scaling laws through iterative dialogue, improving both
accuracy and interpretability on property prediction. That validator is itself a language model, so its
verdicts inherit the confabulation they are meant to police and cannot return an authoritative reference
value; ours come from frozen database records and physical law, which is what makes them auditable and what
supplies the repair signal. Learned grounded step-rewards \cite{lightman2024verify} supply soft signals over
reasoning steps, where we issue hard verdicts on locally checkable objects. Finally, verifier-in-the-loop
generation is established in fully formalizable domains \cite{moura2021lean,trinh2024alphageometry} where
every step is machine-checkable; our contribution is to carry the paradigm into a partially,
tiered-verifiable domain in which only some claims have authoritative checks, and to show where that partial
verifiability does and does not pay off. The distinctive contributions are the detect/repair decomposition
that localizes the bottleneck to detection, the reference-frame-artifact analysis for quantitative
properties, and the object-versus-derived-quantity scope boundary.

Our detect/repair diagnosis has a direct analogue in a different augmentation paradigm. Skill Retrieval
Augmentation \cite{su2026sra} decomposes agentic skill use into retrieval, incorporation and application,
and finds the binding constraint is incorporation, whether the model recognizes when and which skill to
load; our result has the same shape in the grounding paradigm. Both point to need-aware gating and
precision calibration as the next lever rather than to larger retrieval corpora. Relative to
document-grounded factuality checking \cite{min2023factscore}, which scores atomic claims by entailment
against retrieved passages, our verifier issues hard verdicts against structured databases and physics
rather than soft entailment against prose, trading breadth for auditability and for a repair signal that
carries the reference value. Structured extraction components are complementary and partly already inside
the pipeline: RDKit canonicalization backs the Tier-0 SMILES check, and name-to-structure parsers such as
OPSIN \cite{lowe2011opsin} are the route to widening the extractability axis our results identify as
limiting. More broadly, the survey of verification-enabled mathematical reasoning
\cite{raiyan2026ai4math} frames these workflows for fully formalizable domains; we carry them into
partially verifiable science, and our efficiency findings support the case for specialized,
tool-integrated, auditable pipelines argued from the domain-specific-superintelligence perspective
\cite{belova2026dss}.

\subsection*{Limitations}
Extraction recall bounds the method: claims the extractor misses cannot be checked, and in-loop recall is substantially below offline recall. Because the formation-energy detector shares extraction logic with the adjudicator, their blind spots may be correlated; the independent ground-truth audit provides the control. Verification is also limited to errors that are both present in the base model and extractable, with the Claude Haiku 4.5 case showing that formatting alone can suppress the benefit. Correction introduces regeneration noise, and confidence-weighted gating has yet to be validated on multi-claim surfaces. The hybrid two-stage tier improves recall but has not yet been evaluated for accuracy because it does not inject a reference. Gate timing and anchor phrasing were not varied. Finally, the dipole surface is underpowered ($n=56$), the frontier-model isotope lift does not survive multiple-comparison correction, and only the headroom transfer domain shows an accuracy gain; these results therefore support mechanism portability rather than confirmatory performance claims.

Outlook. The main opportunities are improving detection and extraction rather than expanding retrieval. Calibrated detector fusion with trust-based gating and reference injection could convert recovered recall into an accuracy gain, while formatting-robust extraction and name-to-structure parsing (e.g., OPSIN) would expand the checkable surface. Varying gate timing and re-elicitation anchors could further improve in-loop recall, and larger confirmatory runs would resolve the underpowered surfaces. Together, these directions make extraction coverage and detection precision the primary levers for extending grounded verification.

\section{Methods}\label{sec:methods}

\subsection*{Corpus} The corpus comprises 528 condition-pinned prompts across 4 models (versions and access strings in Supplementary Table~\ref{tab:models}) and four rarity strata ordered by PubChem CID, with both molecular and crystalline targets within each stratum and ground truth frozen and hashed before grading. We label the four strata R1, the well-known
entities, through R4, the rare or recently reported ones, and use those labels throughout; the
S-prefixed labels denote Supplementary Notes and Tables only.

\subsection*{Verifier} Tiered: Tier~0 syntactic (SMILES/formula hygiene) $\rightarrow$ Tier~1 identity
(PubChem formula, named-phase space group) $\rightarrow$ Tier~1.5 tabulated
property (Materials Project formation energy/band gap, CCCBDB
dipole) with accept-either-frame policy $\rightarrow$ Tier~3 DFT
(def2-TZVP; last resort). xTB Tier~2 \cite{bannwarth2019gfn2} descoped: the property
surface is fully covered by Tier~1.5 + frame policy; no corpus claim type requires semi-empirical
QM. The extractor and full verifier configuration are frozen and registered ex ante with a content hash, and every extension beyond this core (detection tiers, transfer domains) was likewise registered ex ante; the frozen extractor and verifier configs were never modified.

\subsection*{Conditions} (1) unguarded baseline, (2) self-critique no-tools (with a
stronger tool-free variant, Chain-of-Verification, reported with the Results), (3) RAG-in-prompt (oracle),
run in two variants, conversational (facts prepended, answer free-form) and constrained (facts
prepended with an instruction to restate the reference value as the final answer), (4) gated
Mode~A (verifier-guided correction, retrieve only on flag), (5) Mode~B best-of-8 rerank. The
constrained variant instantiates the extraction-enforcing RAG family and isolates how much of the
conversational baseline's no-commit rate is prompt-shaping rather than intrinsic; both variants draw
the same oracle fact block.

\subsection*{Extraction pipeline} Detection is bounded by what the extractor binds, so we make its
operation explicit (the full verify-and-correct flow is summarized in the method overview figure). Each trace is normalized before ASCII regexes extract molecular formulas, SMILES, space-group
symbols/numbers, and property values with units. Names are bound to claims by whole-word matching
(so ``benzene'' does not match inside ``nitrobenzene''), and values are bound to the nearest named
subject. The frozen materials extractor deliberately does \emph{not}
normalize Unicode-subscript formulas (C$_9$H$_8$O$_4$): these become no-commit under identical rules
in every arm rather than silent errors. Known failure classes (Unicode subscripts, unit-less bare
values, terse final answers with no adjacent subject) are the extractability limit; per-object-type extraction rate and accuracy are in Supplementary Table~\ref{tab:extraction}, and the full
extractor is in the released code.

\subsection*{Independent ground truth} Baseline error is scored against PubChem, Materials Project,
and CCCBDB directly rather than by running the verifier, which controls for the verifier's own blind
spots: the detection tier and the adjudicator share extraction logic, so a claim the extractor cannot
bind would otherwise be invisible to both the intervention and its evaluation. Because extraction rate
is the ceiling on verifier recall, this independent scoring is what separates a genuine reduction in
error from an apparent one produced by the extractor declining to commit.

\subsection*{Parser-artifact retro-audit} Motivated ex ante by the constants-domain finding,
we tested whether the materials lift could be an extraction
artifact of the same kind, the asymmetric bug in which a parser scores a digit-correct baseline
wrong and Mode~A's re-elicitation flips it correct, manufacturing lift. The materials extractor (frozen core) behaves differently: as described under the extraction pipeline, unnormalized Unicode-subscript formulas become no-commit/abstain, scored as abstention rather than error, under identical rules in every arm. The decisive test is therefore whether Mode~A converts these subscript-abstains to parseable commitments at a
higher rate than baseline, which would inflate the lift. It does not: the subscript-only miss rate
is symmetric across arms: baseline 23.8\% vs Mode~A
24.1\%, so the abstain behavior cancels and the \errcommit\ lift
arises from genuinely changed committed formulas, not parse
differences. The two denominators differ by design: the baseline is a single unguarded condition,
$n=564$; Mode~A pools the three tier-active replicates. The artifact concern has two
channels, each with its own control. The abstain channel, a correct baseline value left
unparsed and scored no-commit, is closed by the symmetry test above. The wrong-parse
channel, a correct baseline value misparsed into a different value and scored wrong,
is closed by the manual raw-trace audits of every flagged baseline error. The materials headline is not parser-manufactured. Limitation: the registered per-cell classification of a repaired-cell sample into genuine versus parse-artifact requires round-0 committed text, which was not persisted (only the round-0 verdict and final text); the population tests above close both channels at higher power.

\subsection*{Statistics} Aggregated-unit McNemar \cite{mcnemar1947} (majority outcome per
prompt$\times$model over three reps); prompt-clustered bootstrap 95\% CIs; three-outcome accounting
throughout. Per-replicate rates in supplementary. We state the dependence structure explicitly.
Per-model McNemar tests are the primary inference: within a model, the prompt$\times$model units are
distinct prompts and independent. Pooled (across-model) tests share the same prompt set across the
four models, so their units are not fully independent and their $p$-values are reported as
descriptive. To confirm the pooled headlines are not artifacts of that dependence, we add a
prompt-clustered sensitivity analysis: a cluster bootstrap of the net discordant count that resamples
whole prompts (all model rows for a prompt together, 10{,}000 resamples), reported per headline in
Supplementary Table~\ref{tab:robust}. Both significant pooled headlines, molecular formula and formation energy, survive with no resample reaching net~$\leq 0$, while crystalline space group is non-significant under clustering as in the primary test. No pooled significance required demotion.
Across the thirteen McNemar tests that carry a reported $p$-value, we control for multiple
comparisons with the Holm and Bonferroni procedures at $\alpha=0.05$. The primary headlines survive both corrections: molecular-formula ITT, isotope overall, Mode~B versus Mode~A on formation energy, the end-task derived-tier per-model Qwen result, the formation-energy in-loop test. Re-estimating these two tests at five
replicates rather than three sharpens them to $p=0.016$ (Claude Sonnet~5) and $p=0.008$ (GPT-5.5);
GPT-5.5 then crosses the Holm threshold but not Bonferroni, and Claude Sonnet~5 crosses neither, so
both remain reported as suggestive.

\subsection*{Transfer domain (physical constants)} Reference: CODATA 2022,
with ground truth frozen and hashed before grading. Grading exploits the exactness structure of
the 2019 SI revision: seven defining constants ($c$, $h$, $e$, $k_{\mathrm{B}}$, $N_{\mathrm{A}}$,
$\Delta\nu_{\mathrm{Cs}}$, $K_{\mathrm{cd}}$) and the quantities exactly derived from them ($R$,
$F$, $\sigma$, Josephson and von Klitzing constants, etc.) are exact and graded by
significant figures: a truncated but digit-correct value passes, an altered digit fails.
Measured constants are graded within the larger of the CODATA-2022 standard uncertainty and the
2018--2022 adjustment drift ($\le 4\times10^{-9}$ for every measured constant here), so neither
adjustment version nor quoted-digit truncation can register as an error. Standard-state-dependent
constants (Sackur--Tetrode) are pinned to the modern 100~kPa convention. The value extractor was
hardened after the constants-domain audit found the first-pass parser dropped exponents.


\begin{acknowledgments}
No funding was received for this research.
\end{acknowledgments}

\begin{contributions}
C.P. conceived and implemented the tiered grounded-verification pipeline, conducted all
computational experiments and benchmark runs, performed the data analysis, and wrote the original
manuscript. M.K, E.S., and H.K. supervised the research, contributed to the interpretation of
results, and reviewed and edited the manuscript. All authors read and approved the final
manuscript.
\end{contributions}

\begin{conflicts}
The authors declare no competing financial or non-financial interests.
\end{conflicts}

\begin{availability}
The condition-pinned prompt corpus, the frozen ground truth for all domains, and all experimental
results reported in this study are available with the code in the project repository at
\url{https://github.com/KurbanIntelligenceLab/grounded-matsci}, each with a per-file SHA-256
manifest. The corpus and ground-truth data are released under the Creative Commons Attribution 4.0 license, and the verification pipeline and analysis code under the MIT license.

The complete verification pipeline and analysis code are available at
\url{https://github.com/KurbanIntelligenceLab/grounded-matsci}.
\end{availability}

\subsection*{Use of AI tools}
A large language model was used to assist with manuscript drafting, copy editing, formatting, and code development. The authors independently reviewed and verified all scientific statements, proofs, analyses, code, figures, tables, and reported results, and take full responsibility for the content of the manuscript.

\subsection*{Computing resources}
All model generations were obtained through the OpenRouter API. The four primary models were Claude
Sonnet~5, GPT-5.5, DeepSeek~v4~Pro, and Qwen3-235B; two smaller models, Claude Haiku~4.5 and GPT-5.4
mini, were used only for the small-model detection comparison. Decoding used temperature $0$ (greedy)
for the single-answer conditions and temperature $0.7$ for sampled generations, including the
best-of-8 rerank ($N=8$); top-$p$ was left at the provider default. The maximum-token cap depended on
the workflow: $1{,}200$ tokens for the main arms and holdout harness, $900$ for the closed-loop
correction driver used by the two-stage and isotope experiments, and $1{,}400$ for the end-task
runs. Density-functional theory reference calculations and all database lookups
(PubChem, Materials Project, CCCBDB, IAEA NDS) ran on a single workstation CPU.

\subsection*{Reproducibility}
All reported results follow a one-shot holdout protocol: ground truth,
prompt corpus, extractor, verifier configuration, and grading tolerances were frozen and hashed before
grading, and no holdout result was iterated on. Each condition was run in triplicate, with pooled
headline significance tests re-estimated under prompt clustering and per-model tests corrected for
multiple comparisons. Every released artifact carries a per-file SHA-256 manifest and every
experimental condition carries its registration hash, so each reported number is traceable to the run
that produced it. The provider-side model-version snapshot identifiers required to reproduce the
generations are listed in the supplementary model-versions table.

\bibliography{references}

\newpage
\appendix
\setcounter{table}{0}
\renewcommand{\thetable}{S\arabic{table}}
\renewcommand{\theHtable}{S\arabic{table}}
\setcounter{figure}{0}
\renewcommand{\thefigure}{S\arabic{figure}}
\renewcommand{\theHfigure}{S\arabic{figure}}

\newcounter{supnote}
\renewcommand{\thesupnote}{S\arabic{supnote}}
\newcommand{\supnote}[2]{%
  \refstepcounter{supnote}%
  \subsection*{Supplementary Note~\thesupnote: #2}%
  \label{#1}%
}

\section{Supplementary information}\label{secA1}

This section collects the extended tables, robustness analyses, and notes that support the main-text claims but are not needed to follow the argument: the baseline error structure, the per-condition compute accounting and the retrieval-savings breakdown, a realistic non-oracle retrieval bound, a proof of the error-reduction identity, a Chain-of-Verification tool-free baseline, a reference-frame-policy robustness check, a four-detector comparison, the trust-score calibration analysis in full, a parser-artifact audit of the physical-constants domain, the isotope frontier lift re-estimated at five replicates, the model versions and access details, the per-object-type extraction rates, and a clustered re-analysis of the headline significance tests. Each item is referenced from the main text at the point it supports; the sections below state what each one shows and its single takeaway.

\subsection*{Error structure}
Table~\ref{tab:errors} shows where the 248 genuine baseline errors actually fall. Two patterns
determine where correction can help: errors concentrate in the long-tail rarity strata (R3--R4)
rather than the well-known compounds, and they concentrate in the weakest model (Qwen3-235B alone
accounts for 131 of 248). This is the error structure behind the per-model lift in
Table~\ref{tab:permodel}: correction lifts accuracy exactly where errors are dense and leaves the
near-ceiling frontier models roughly unchanged.

\begin{table}[htbp]
\caption{Genuine baseline errors by claim type $\times$ stratum (pooled models) and by model. 248
genuine baseline errors (correct = false), concentrated in the long-tail strata and the weakest
model, the pattern that governs where correction lifts accuracy
(Table~\ref{tab:permodel}).}\label{tab:errors}
\footnotesize
\begin{tabular*}{\textwidth}{@{\extracolsep{\fill}}lccl}
\toprule
claim type & total errors & leading strata (count) & representative subject \\
\midrule
formation energy & 116 & R4 (49), R3 (35), R2 (22) & halite (NaCl) \\
molecular formula & 98 & R4 (69), R3 (27) & ampicillin \\
crystalline space group & 16 & R3 (6), R4 (6) & sphalerite (ZnS) \\
reasoning system & 11 & R4 (11) & Cs$_2$AgBiBr$_6$ perovskite \\
dipole & 7 & R1 well-known (7) & nitrobenzene \\
\bottomrule
\end{tabular*}
\tabnote{By model: Qwen3-235B 131, DeepSeek v4 Pro 53, Claude Sonnet 5 46, GPT-5.5 18.}
\end{table}

\subsection*{Compute cost}
Table~\ref{tab:compute} accounts for the compute of every condition on the same footing: calls per
cell and retrieval rounds are measured, USD and wall-clock are scaled from the measured base
generation cost. The takeaway is the deployment trade between the two interventions: best-of-8
rerank (Mode~B) costs ${\sim}8\times$ baseline generation (${\sim}\$212$, ${\sim}92$~min) for its
formation-energy gain, whereas gated correction (Mode~A) costs only ${\sim}1.75\times$
(${\sim}\$46$, ${\sim}20$~min) because it retrieves only on a flag rather than on every prompt.
Fig.~\ref{fig:cost} visualizes the retrieval side of this saving.

\begin{table}[htbp]
\caption{Per-condition compute cost for the full evaluation (528 prompts $\times$ 4 models $\times$ 3
replicates = 6{,}336 cells per condition). Calls per cell and retrieval rounds are measured; USD and
wall-clock are scaled from the measured base generation cost (\$0.00418/trace) and measured call
multipliers, at 12 concurrent workers ($\sim$1.3~s/call). RAG variants issue one blanket retrieval
per prompt at deployment (100 per 100); gated Mode~A retrieves only on a flag (17 per 100, the
measured overall prompt$\times$model cell flag rate). Per-surface breakdown in Fig.~\ref{fig:cost}.}\label{tab:compute}
\footnotesize
\begin{tabular*}{\textwidth}{@{\extracolsep{\fill}}lccccc}
\toprule
condition & calls/cell & total LLM calls & retrievals/100 & USD & wall-clock (min) \\
\midrule
baseline (unguarded) & 1.00 & 6{,}336 & 0 & 26 & 11 \\
self-critique & 3.00 & 19{,}008 & 0 & 79 & 34 \\
RAG conversational & 1.00 & 6{,}336 & 100 & 26 & 11 \\
RAG constrained & 1.00 & 6{,}336 & 100 & 26 & 11 \\
gated Mode~A & 1.75 & 11{,}075 & 17 & 46 & 20 \\
Mode~B ($N{=}8$) & 8.00 & 50{,}688 & 0 & 212 & 92 \\
\bottomrule
\end{tabular*}
\tabnote{Mode~B's best-of-8 rerank costs ${\sim}8\times$ baseline generation (${\sim}\$212$,
${\sim}92$~min) for its formation-energy gain, versus gated Mode~A's ${\sim}1.75\times$ (${\sim}\$46$,
${\sim}20$~min); this is the deployment trade the reviewer raises. Chain-of-Verification
costs four LLM calls per cell (draft, plan, independent answers, revise), above self-critique's three,
and still closes no surface; the extra tool-free reasoning buys no
accuracy.}
\end{table}

\subsection*{Retrieval savings}
Fig.~\ref{fig:cost} makes the retrieval side of the cost story concrete:
because gated correction retrieves only when a claim is flagged, it issues far fewer external
lookups than blanket RAG's one-per-prompt baseline, $3.2\times$ fewer on molecular formula,
$6.2\times$ on formation energy, and $8.4\times$ on crystalline space group. The three bars report
round-0 Mode~A cell flag rates (triplicate mean) on those surfaces, where the pooled 17 per 100
figure covers all claim types.

\begin{figure}[htbp]
\centering
\includegraphics[width=0.8\textwidth]{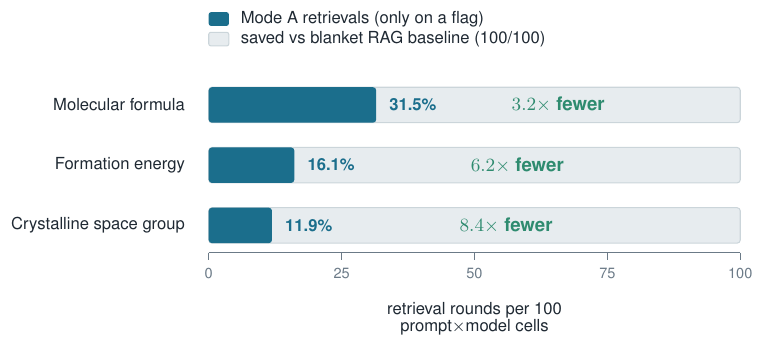}
\caption{Retrieval savings of gated correction. Gated Mode~A retrieves only on a flagged
prompt$\times$model cell, versus blanket RAG's 100/100 baseline (one lookup per prompt at
deployment), from the precise per-surface round-0 cell flag rates: molecular 31.5\% ($3.2\times$
fewer retrievals), formation energy 16.1\% ($6.2\times$), crystalline SG 11.9\% ($8.4\times$). Full
per-condition compute cost is in the compute-accounting table.}\label{fig:cost}
\end{figure}

\supnote{note:nonoracle}{Realistic (non-oracle) retrieval}
The RAG conditions in the main text are an oracle upper bound, drawing facts from the same database
records that define the ground truth, so coverage is complete and every fact is correct. To bound what a
deployed system would see, we built a realistic retriever that resolves each named entity live:
molecular names against PubChem (formula and SMILES), and crystalline and formation-energy entities by
resolving a formula and querying the Materials Project for the most-stable polymorph's space group and
formation energy. No ground-truth record is consulted during retrieval. Retriever quality varies sharply
by surface. Molecular coverage is 141 of 141 entities, all correct, because every compound in the corpus
is name-addressable in PubChem, the same source the molecular ground truth draws on. Crystalline coverage
is 27 of 127 entities (21\%), of which 20 of 27 return a space group within the accepted set; formation-
energy coverage is 23 of 101 entities (23\%), all within tolerance when returned. The two failure modes
on the long tail are a name that does not resolve to a formula, and a formula that resolves to the
most-stable polymorph rather than the phase the prompt names (for example rutile returns space group
No.~141 where the prompt fixes the No.~136 phase). Prepending these retrieved facts and generating once,
realistic retrieval matches the oracle on molecular formula (error given commitment 0.2\% versus the
oracle's 0.3\%, baseline 21.7\%) with low abstention (4.8\% versus the oracle's 42.0\%), but recovers
almost none of the oracle gain on the long tail: error given commitment holds at baseline (formation
energy 33.9\% versus 31.3\%; space group 5.3\% versus 5.8\%) and abstention rises to 56.2\% and 59.4\% as
the model, handed no usable fact, declines to commit. The verifier's claim-specific retrieval, keyed to
the named phase rather than to bare composition, is what still succeeds in this regime.

\supnote{note:identity}{proof of the error-reduction identity}
\emph{Setup.} Consider the claims committed and graded in both the baseline and Mode-A arms on a fixed
surface, with baseline error-given-commitment $e_0=\Pr(\mathrm{wrong})$. Gated correction applies a
fire/hold decision $D$ per claim (the verifier flag conjoined with the Platt gate
$\sigma(as+b)<\theta$, where $s$ is the raw verifier score from the main text and
$\sigma$ is the logistic function); a fired claim is regenerated with the reference value injected. Define the
conditional rates $r_d=\Pr(D{=}\mathrm{fire}\mid\mathrm{wrong})$,
$r_p=\Pr(\mathrm{correct\ after\ regen}\mid D{=}\mathrm{fire},\mathrm{wrong})$,
$\phi=\Pr(D{=}\mathrm{fire}\mid\mathrm{correct})$, and
$\rho=\Pr(\mathrm{wrong\ after\ regen}\mid D{=}\mathrm{fire},\mathrm{correct})$. Held claims are
unchanged, and any regeneration that leaves the committed population (a regenerated no-commit) is
excluded by the committed-in-both-arms restriction, so every claim ends either correct or wrong.

\begin{proof}
Partition the population by baseline correctness. A wrong claim (mass $e_0$) ends correct iff
$D{=}\mathrm{fire}$ and repair succeeds, probability $r_d r_p$; otherwise it ends wrong, contributing
$e_0(1-r_d r_p)$ to the post-correction error. A correct claim (mass $1-e_0$) ends wrong iff
$D{=}\mathrm{fire}$ and regeneration breaks it, probability $\phi\rho$, contributing $(1-e_0)\phi\rho$.
These two events partition the ways a committed claim can be wrong after correction and no other
transition changes correctness, so by the law of total probability
$e_1=e_0(1-r_d r_p)+(1-e_0)\phi\rho$, whence $\Delta=e_0-e_1=e_0 r_d r_p-(1-e_0)\phi\rho$.
\end{proof}

\emph{Precision form of the harm term.} The fired set has mass $F=e_0 r_d+(1-e_0)\phi$ and detector
precision $\pi=e_0 r_d/F$, so $(1-e_0)\phi=(1-\pi)F$ and the harm term equals $(1-\pi)F\rho$. Harm
therefore grows as precision falls: a high-recall, low-precision detector (the consistency-triggered
stage reaches precision $0.28$ on formation energy) fires largely on already-correct claims, and unless
the break rate $\rho$ is suppressed the gain $e_0 r_d r_p$ is offset. Two levers reduce harm without
lowering recall: injecting a reference, which raises $r_p$ toward $1$ and is what the deterministic tier
supplies but the consistency signal does not; and gating on a calibrated trust score, which lowers
$\phi$ and hence raises $\pi$. The identity thus separates the two requirements the experiments
establish: detection recall sets the reachable gain, and a reference-carrying, precision-calibrated
fire sets how much of it is realized without harm.

The identity is an exact accounting relation, following from the law of total probability applied to
the two ways a committed claim can be wrong after correction. As a numerical check on the
molecular-formula surface, substituting $e_0{=}0.22$, $r_d{=}0.82$, and $r_p{=}0.97$ with a
negligible false-positive break term gives $e_1=e_0(1-r_d r_p)=4.5\%$, consistent with the observed
post-correction error of about $4\%$.

\supnote{note:cove}{Chain-of-Verification as a stronger tool-free baseline}
Self-critique is a single-pass no-tools control. To test whether a
stronger structured self-verification method fares differently, we ran Chain-of-Verification
\cite{dhuliawala2024cove}: the model drafts an answer, plans a set of verification questions targeting
its own factual claims, answers those questions independently, then revises. No external tool is used
at any stage. We ran the four-stage protocol on the three error surfaces (molecular formula, formation
energy, crystalline space group) with the four primary models in triplicate, using the same corpus,
system prompt, and frozen grader (final-answer rule) as the self-critique arm; 4{,}427 cells, no API
errors. It closes no surface: on the two surfaces whose answer formatting the frozen grader parses
without ambiguity, error given commitment is essentially unchanged (molecular 21.7\% to 18.8\%,
crystalline 5.8\% to 6.0\%), and on formation energy it does not improve on the baseline. A formatting
note bounds the formation-energy comparison: about a quarter of Chain-of-Verification answers write the
value with a Unicode minus sign, which the frozen extractor's ASCII regex misreads, inflating the raw
error rate; the glyph rate differs across conditions, so we report the raw frozen-grader numbers and do
not renormalize the published self-critique value. Under a sensitivity check that maps the Unicode minus
to ASCII for all conditions, Chain-of-Verification formation-energy error is 29.5\% against a 31.3\%
baseline, still no improvement. The finding matches self-critique: a stronger tool-free protocol does
not reach the gap that tool-grounded correction closes, because the model cannot supply the external
reference value it lacks.

\subsection*{Reference-frame robustness}
A quantitative-property claim can look ``wrong'' merely because the model and the reference use
different conventions (e.g.\ a computed vs an experimental band gap). Table~\ref{tab:frames} checks
this by re-grading under two policies. Only the band-gap surface is frame-sensitive: its error
collapses from 34.6\% to 0.0\% once either documented frame is accepted (107 of 116 naive flags were
correct experimental gaps), which is why band gap is excluded from the property headline.
Formation energy, whose frames nearly coincide, barely moves (33.7\% to 31.3\%) and remains a genuine
error surface; the frame-insensitive surfaces are unchanged.

\begin{table}[htbp]
\caption{Post-audit baseline error by claim type under the two reference-frame policies. Band gap
collapses once either documented frame is accepted (107/116 naive flags were correct experimental
gaps); formation energy, whose frames nearly coincide, barely moves and remains a genuine error
surface; frame-insensitive surfaces are unchanged.}\label{tab:frames}
\footnotesize
\begin{tabular*}{\textwidth}{@{\extracolsep{\fill}}lcc}
\toprule
claim type & naive single-frame & accept-either-documented-frame \\
\midrule
band gap & 34.6\% & 0.0\% \\
formation energy & 33.7\% & 31.3\% \\
molecular formula & 21.7\% & 21.7\% \\
dipole & 12.5\% & 12.5\% \\
space group & 5.8\% & 5.8\% \\
crystal system & 2.7\% & 2.7\% \\
\bottomrule
\end{tabular*}
\end{table}

\subsection*{Why a deterministic checker}
Table~\ref{tab:detectors} compares our deterministic tiered checker against three alternatives on the
same frozen baseline claims: a frontier reasoning-class LLM-as-judge (no tools), a
sampling-consistency baseline (SelfCheckGPT-style), and a semantic-entropy detector.
The deterministic checker gives the high recall
that matters for catching errors on the molecular-formula surface (0.92) while keeping precision
usable (0.51), where the alternatives either trade recall for precision (the LLM judge, 0.71 recall
at 0.86 precision) or over-flag (sampling-consistency, 0.85 recall but only 0.40 precision, with many
false alarms); no single detector dominates on every surface, but only the deterministic tier
returns an external reference value to inject, which is what the correction step requires.
The semantic-entropy detector, which clusters the eight
samples by meaning before measuring entropy, behaves as a refinement of the
sampling-consistency baseline: it matches its recall on molecular formula (0.93) and raises
formation-energy precision to 0.51 (from 0.36) at slightly lower recall (0.57 from 0.67), but like it
returns no reference value and so cannot drive the correction step.

\begin{table}[htbp]
\caption{Four-detector precision (P) and recall (R) by surface (offline, frozen baseline claim text):
the deterministic verifier, a frontier reasoning-class LLM-as-judge (no tools) \cite{zheng2023judge},
a sampling-consistency baseline (SelfCheckGPT-style \cite{manakul2023selfcheckgpt}, flag if
modal-object agreement $\le 0.25$, threshold frozen on a dev half), and a semantic-entropy
detector \cite{farquhar2024semantic} that clusters the same eight samples by meaning and flags high
normalized cluster entropy (threshold $0.44$ frozen on a dev half). Bold marks the highest recall per
surface; only the deterministic verifier also returns a reference value to inject.}\label{tab:detectors}
\footnotesize
\setlength{\tabcolsep}{3pt}
\renewcommand{\arraystretch}{1.18}
\begin{tabular*}{\textwidth}{@{\extracolsep{\fill}}l cc cc cc cc@{}}
\toprule
& \multicolumn{2}{c}{\textbf{Deterministic}} & \multicolumn{2}{c}{\textbf{LLM-as-judge}} & \multicolumn{2}{c}{\textbf{Sampling-consistency}} & \multicolumn{2}{c}{\textbf{Semantic entropy}} \\
\cmidrule(lr){2-3}\cmidrule(lr){4-5}\cmidrule(lr){6-7}\cmidrule(l){8-9}
surface & P & R & P & R & P & R & P & R \\
\midrule
molecular formula & 0.51 & 0.92 & 0.86 & 0.71 & 0.40 & 0.85 & 0.31 & \textbf{0.93} \\
formation energy & 0.55 & 0.31 & 0.53 & 0.53 & 0.36 & \textbf{0.67} & 0.51 & 0.57 \\
crystalline SG & 0.20 & \textbf{0.69} & 0.67 & 0.38 & 0.17 & 0.17 & 0.27 & 0.20 \\
\bottomrule
\end{tabular*}
\tabnote{P, precision; R, recall. The sampling-consistency detector is strongest exactly where
symbolic extraction is weakest (formation energy), but returns no reference value, so it can detect but
not repair. Semantic entropy, computed by reanalysis of the stored eight-sample sets
with no new generation, tracks the sampling-consistency detector while improving formation-energy
precision ($0.51$ versus $0.36$) through meaning-based clustering; like it, semantic entropy is
detection-only and returns no reference value to inject.}
\end{table}

\supnote{note:calibration}{Trust-score calibration in full}
The verifier's trust score (fraction of checked claims passing) predicts committed-claim correctness
on $n=968$ baseline items; on a held-out half ($n=484$, base correct rate 0.81) raw trust is
overconfident (expected calibration error, ECE, 0.111) \cite{guo2017calibration}. Platt scaling \cite{platt1999} reduces
ECE to 0.060 (isotonic 0.088); Brier 0.121 to 0.111. The trust score is discrete (few checkable
claims per trace), yielding a coarse reliability curve; ECE is the reportable summary and Platt is
recommended for deployment thresholding. Against model verbalized confidence, each
generating model asked to self-report P(its committed claim correct), the grounded trust score is
better calibrated. On the identical paired population ($n=958$, items carrying both signals), raw
uncalibrated: verbalized confidence ECE 0.139 (overconfident: mean 0.87, accuracy 0.83) vs verifier
trust score ECE 0.093. Recalibrating the trust score by Platt on a separate holdout half drops it
further to 0.060 ($n=484$; not directly comparable to the paired set, which is uncalibrated). The
models' own confidence is a poorly-calibrated abstention signal; the external verifier's grounded
trust score is better even before recalibration and substantially better after. Eliciting confidence
inline with the answer (a dedicated calibration arm, $n=1{,}191$) makes the model
more confident (mean 0.91 vs 0.87 post-hoc) without improving calibration (ECE 0.108); both
elicitation modes remain worse-calibrated than the grounded trust score
(Fig.~\ref{fig:calibration}).

\begin{figure}[htbp]
\centering
\includegraphics[width=\textwidth]{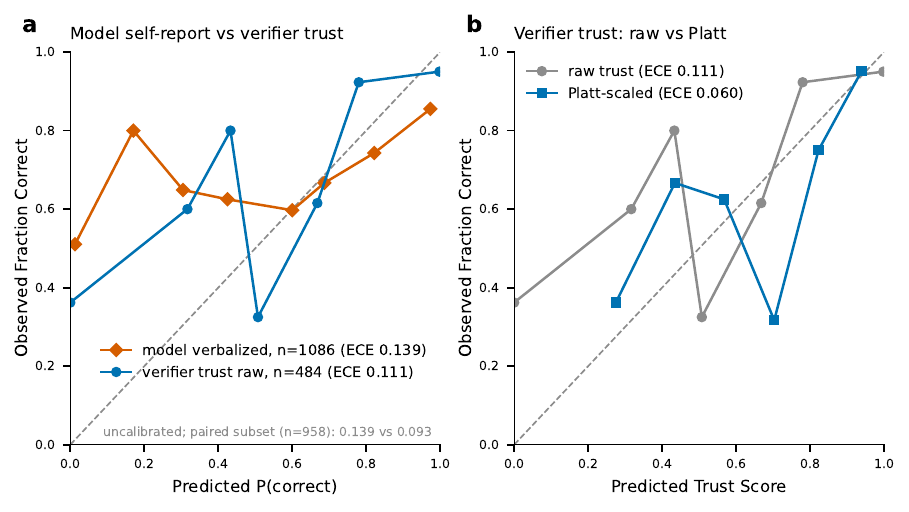}
\caption{Two-panel reliability diagram. \textbf{a}, Model verbalized confidence ($n=1{,}086$) vs
verifier trust score, raw and uncalibrated ($n=484$), each on its own population, with the
paired-subset comparison ($n=958$); the grounded trust score is the better-calibrated abstention
signal on every population (ECE values in text). \textbf{b}, Verifier trust score, raw vs
Platt-scaled, on the $n=484$ trust-only holdout.}\label{fig:calibration}
\end{figure}
\supnote{note:constants}{Physical-constants audit (parser artifact and per-model detail)}
The first-pass constants result (84\% to 99\%) was traced by a line-by-line audit of every flagged
constant: ${\sim}23$ of ${\sim}29$ baseline ``errors'' were orchestration-layer extraction
artifacts: the model stated the correct value in scientific notation
($1.602176634\times10^{-19}$, $4\pi\times10^{-7}$, or LaTeX-brace form) but the value
parser dropped the exponent, and Mode~A ``fixed'' these by eliciting an ASCII-parseable restatement
of the same value. After hardening the parser (Unicode superscripts, LaTeX braces,
symbolic-$\pi$, digit-group spacing) and re-grading the full triplicate under the registered
tolerance policy, no model changes significantly: the two frontier models start at 100\% baseline
and dip slightly under Mode~A (Claude Sonnet 5 100\% to 97\%, GPT-5.5 100\% to 94\%) from regeneration noise
on already-correct answers, while the open-weights models each gain ${\sim}5.7$ points (DeepSeek v4 Pro
91.4\% to 97.1\%, Qwen3-235B 88.6\% to 94.3\%); all $p\ge0.5$. The verifier still detects the
genuinely wrong constants that remain (Qwen3-235B dropping the $\times10^{-8}$ on the Stefan--Boltzmann
constant, a missing sign on the electron $g$-factor, a $10^{3}$ unit-scale slip on the Josephson
constant). This asymmetric artifact (baseline scored wrong, re-elicitation flipped it to
correct) is the failure mode the materials retro-audit checked for and, via the symmetric
subscript-abstain test, ruled out for the materials headline.

\supnote{note:isotope}{Isotope frontier lift at five replicates}
The two per-model frontier isotope lifts reported in the main text (Claude Sonnet~5 and
GPT-5.5) rest on triplicate McNemar tests with $p=0.016$ and $p=0.021$, which do not survive correction
for multiple comparisons. Because these borderline claims carry the paper's only frontier-model lift, we
re-estimated the isotope arm at five replicates. We generated two additional replicates of both arms
(baseline and gated Mode~A) over the 63 isotope prompts and four primary models (1{,}008 cells, no API
errors), using the same prompts, ground truth, system prompt, tolerance, and per-model aggregation
(correct versus not-correct on model-by-subject units, majority over replicates) as the triplicate
analysis. The per-model net lifts are unchanged (Claude Sonnet~5 net $+7$, GPT-5.5 net $+8$, Qwen3-235B
net $+11$; DeepSeek v4 Pro rises from $+4$ to $+6$), and the two frontier $p$-values sharpen to
$p=0.016$ (Claude Sonnet~5) and $p=0.008$ (GPT-5.5). Under the same thirteen-test Holm and Bonferroni
family, GPT-5.5 then crosses the Holm threshold but not Bonferroni, and Claude Sonnet~5 crosses neither.
The added replicates therefore tighten the estimates without changing the conclusion, and we continue to
report the frontier lift as suggestive rather than confirmatory.

\subsection*{Model versions and reproducibility}
All generations were served through the OpenRouter API. Table~\ref{tab:models} records the exact
model string for each model so that the runs can be reproduced against the same provider-side snapshots.

\begin{table}[htbp]
\caption{Model versions. Snapshot identifiers are the OpenRouter model strings used at run time.}\label{tab:models}
\footnotesize
\begin{tabular*}{\textwidth}{@{\extracolsep{\fill}}lll}
\toprule
model & role & OpenRouter model string \\
\midrule
Claude Sonnet 5 & primary & \texttt{anthropic/claude-sonnet-5} \\
GPT-5.5 & primary & \texttt{openai/gpt-5.5} \\
DeepSeek v4 Pro & primary & \texttt{deepseek/deepseek-v4-pro} \\
Qwen3-235B & primary & \texttt{qwen/qwen3-235b-a22b-2507} \\
Claude Haiku 4.5 & small (detection only) & \texttt{anthropic/claude-haiku-4.5} \\
GPT-5.4 mini & small (detection only) & \texttt{openai/gpt-5.4-mini} \\
\bottomrule
\end{tabular*}
\tabnote{Hardware: all model queries were issued through the OpenRouter API from a
single Apple M4 Pro MacBook Pro (24~GB unified memory), which also ran the local DFT reference
calculations and database lookups.}
\end{table}

\subsection*{Extraction rates and the recall ceiling}
The verifier can only correct what its extractor first commits to a parseable claim, so we report
the per-object-type extraction rate alongside accuracy. Table~\ref{tab:extraction} gives both: extraction
rate (the fraction of cells that commit a parseable value) sets a hard ceiling on verifier recall,
which is why the crystalline space-group surface, committed in only 54\% of cells, has a low
detection recall regardless of how good the adjudicator is.

\begin{table}[htbp]
\caption{Frozen extractor holdout extraction rate and accuracy by object type. Extraction rate =
fraction of cells committing a parseable claim ($1-$ no-commit); accuracy $\vert$ commit =
correctness among committed. Baseline unguarded condition, frozen extractor. The extraction rate is
the ceiling on verifier recall.}\label{tab:extraction}
\footnotesize
\setlength{\tabcolsep}{4pt}
\renewcommand{\arraystretch}{1.18}
\begin{tabular*}{\textwidth}{@{\extracolsep{\fill}}l cc cc@{}}
\toprule
object type & $n$ & committed & extraction rate & accuracy $\vert$ commit \\
\midrule
molecular formula & 564 & 451 & 80.0\% & 78.3\% \\
crystalline space group & 508 & 275 & \nullv{54.1\%} & 94.2\% \\
reasoning system (crystal system) & 404 & 404 & 100.0\% & 97.3\% \\
formation energy & 404 & 371 & 91.8\% & 68.7\% \\
band gap & 176 & 175 & 99.4\% & 100.0\% \\
dipole & 56 & 56 & 100.0\% & 87.5\% \\
\bottomrule
\end{tabular*}
\tabnote{Extraction rate is the ceiling on verifier recall: the space-group tier's 54\% commit
rate (grey) caps its detection recall regardless of adjudication quality.}
\end{table}

\subsection*{Robustness of the headline statistics}
The pooled McNemar tests in the main text treat prompt$\times$model cells as independent, but the
three replicates of each prompt are not. Table~\ref{tab:robust} re-tests the two significant headlines
under a prompt-clustered cluster bootstrap that resamples whole prompts (all four model rows
together): both the molecular-formula and formation-energy net effects remain strictly positive
across 10{,}000 resamples (fraction with net~$\leq0$ is 0.000), and crystalline space group stays
non-significant, so the headlines are not artifacts of treating replicates as independent.

\begin{table}[htbp]
\caption{Prompt-clustered sensitivity of the pooled Mode-A McNemar headlines. Cluster bootstrap of
the net discordant count (lift $-$ harm), resampling whole prompts (all four model rows together),
10{,}000 resamples. ``frac.\ net~$\leq 0$'' is the fraction of resamples with non-positive net (a
one-sided robustness check). Both significant headlines survive clustering; crystalline is
non-significant as in the primary per-model test.}\label{tab:robust}
\footnotesize
\begin{tabular*}{\textwidth}{@{\extracolsep{\fill}}lcccccc}
\toprule
surface & units & prompts & lift & harm & net (95\% CI) & frac.\ net~$\leq 0$ \\
\midrule
molecular formula & 444 & 141 & 86 & 6 & $+80$ $[62, 99]$ & 0.000 \\
formation energy & 369 & 101 & 45 & 18 & $+27$ $[14, 40]$ & 0.000 \\
crystalline space group & 258 & 102 & 5 & 7 & $-2$ $[-10, 5]$ & 0.75 \\
\bottomrule
\end{tabular*}
\tabnote{Units = prompt$\times$model pairs with a gradable value in both arms; baseline (single)
vs Mode~A (majority over three reps), joined by subject$\times$model. Clustering is by prompt.}
\end{table}

\end{document}